\documentclass{article}

\PassOptionsToPackage{numbers, compress}{natbib}

\usepackage[preprint]{neurips_2026}

\usepackage[utf8]{inputenc} 
\usepackage[T1]{fontenc}    
\usepackage[colorlinks,citecolor=blue]{hyperref}       
\usepackage{url}            
\usepackage{booktabs}       
\usepackage{amsfonts}       
\usepackage{nicefrac}       
\usepackage{microtype}      
\usepackage[dvipsnames]{xcolor}         
\usepackage{amsmath}
\usepackage{amsthm}
\usepackage[linesnumbered,ruled,lined,noend]{algorithm2e}
\usepackage{mathtools}
\usepackage{subcaption}
\usepackage{amssymb}
\usepackage{bm}
\usepackage{multirow}
\usepackage{tikz}
\usepackage{wrapfig}
\usepackage{siunitx}
\usepackage{makecell}
\usetikzlibrary{shapes.geometric}

\SetKwComment{Comment}{// }{}

\DeclareMathOperator{\1}{\mathbf{1}}

\newtheorem{theorem}{Theorem}

\newtheorem{lemma}{Lemma}
\newtheorem{corollary}{Corollary}

\theoremstyle{remark}
\newtheorem{remark}{Remark}

\title{Parallelizing Counterfactual Regret Minimization}

\author{%
  Juho Kim \\
  Computer Science Department \\
  Carnegie Mellon University \\
  \texttt{juhok@cs.cmu.edu} \\
  \And
  Tuomas Sandholm \\
  Computer Science Department, CMU \\
  Strategic Machine, Inc. \\
  Strategy Robot, Inc. \\
  Optimized Markets, Inc. \\
  \texttt{sandholm@cs.cmu.edu} \\
}

\begin{document}

\maketitle

\begin{abstract}
	Parallelization has played an instrumental role in the field of artificial intelligence (AI), drastically reducing the time taken to train and evaluate large AI models.
	In contrast to its impact in the broader field of AI, applying parallelization to computational game solving is relatively unexplored, despite its great potential.
	In this paper, we parallelize the family of counterfactual regret minimization (CFR) algorithms, which were central to important breakthroughs for solving large imperfect-information games.
	We present a generalized parallelization framework, reframing CFR as a series of linear algebra operations.
	Then, existing techniques for parallelizing linear algebra operations can be applied to accelerate CFR.
	We also describe how our technique can be applied to other tabular members of the CFR family of algorithms, including the state-of-the-art, such as CFR\textsuperscript{+}, discounted CFR, and predictive variants of CFR.
	Experimentally, we show that our CFR implementation on a GPU is up to four orders of magnitude faster than Google DeepMind OpenSpiel's CFR implementations on a CPU.
\end{abstract}

\section{Introduction}

Historically, \textit{parallelization} served as a key enabler in the field of artificial intelligence (AI)~\cite{deaneteal,goyaletal,sergeevanddelbalso}, dramatically reducing the time taken for researchers to train large AI models on massive datasets by splitting the problem into smaller subproblems to be solved concurrently.
This drastic reduction in the training time also facilitated much faster hypothesis testing compared to when serial execution is used.
In hindsight, the explosive growth in the performance and capability of AI models in a myriad of real-life applications would have been impossible without parallelization.

In stark contrast to its broad impact in AI in general, parallelization is relatively unexplored in the area of computational game solving, although the need for parallelism is clearly evidenced by the fact that past landmark works on solving large games, such as Cepheus~\cite{tammelinetal2015}, Libratus~\cite{brownandsandholm2017}, and Pluribus~\cite{brownandsandholm2018}, all leveraged some form of parallelism.
Due to the sheer size of the games involved, developing these AI agents would have been infeasible without parallelization.
Additionally, there are ample reasons to apply parallelization even for more regular game-solving tasks.
Indeed, we later show that many games, even those that are commonly used as benchmarks for game-solving algorithms, provide substantial room for parallelization, using which speedup of \textbf{multiple orders of magnitude} can be achieved.
Thus, leveraging this would allow researchers in the field to obtain much faster turnaround times in testing their theoretical results, hypotheses, or hyperparameter regimes.

In computational game theory, the family of \textit{counterfactual regret minimization (CFR)} algorithms~\cite{zinkevichetal2007} has been used to develop AI agents for large-scale imperfect-information games, most notably poker~\cite{tammelinetal2015,moravciketal2017,brownandsandholm2018,brownandsandholm2019}.
Variants of CFR have since been proposed to improve its convergence rate~\cite{tammelin2014,brownandsandholm2019aaai,farinaetal2021,xuetal2024,zhangetal2026b,zhangetal2026a},
leverage Monte Carlo runouts~\cite{lanctotetal2009}, and incorporate deep learning~\cite{brownetal2019,mcaleeretal2023}.
In this paper, we propose a generalized parallelization framework, which reconsiders CFR and its variants as a series of linear algebra operations.
Then, we can apply standard techniques for parallelizing linear algebra operations to parallelize CFR and its variants.
While our framework does not necessarily increase the size of the game that can be solved, we experimentally show that it can make finding an optimal solution to these games dramatically faster.
Our experiments involving seven games of varying sizes show that our parallelized implementation on a GPU is up to about 18,889 and 3,413 times faster than Google DeepMind OpenSpiel's Python and C++ implementations on a CPU, respectively.
The speedup becomes more pronounced as the size of the game being solved grows.

\section{Notation and background}

This section defines the notation used in this paper and reviews CFR.

\subsection{Sequence-form CFR}

In this paper, we consider parallelizing the sequence-form~\cite{vonstengel,kolleretal,farinaetal2019} version of CFR, as opposed to its classical version~\cite{zinkevichetal2007}.
Note that classical CFR and sequence-form CFR are equivalent in the sense that they produce identical iterates, so parallelizing sequence-form CFR also entails indirectly parallelizing classical CFR.
We will give more details about this choice later in our discussions.

In the sequence-form version, each player runs CFR to minimize the counterfactual regret of their respective \textit{tree-form sequential decision processes (TFSDPs)}, which are derived from the game tree.
In a TFSDP, every node $p \in \mathcal{P}$ is either in the set of decision points $\mathcal{J}$, the set of observation points $\mathcal{K}$, or is an end of the decision process $\bot$.
At a decision point $j$, an available action $a \in \mathcal{A}_j$ can be applied, and $\rho(j, a)$ is used to denote the child node reached by applying $a$ at $j$.
Similarly, a signal $s \in \mathcal{S}_k$ can be observed at an observation point $k$, and the corresponding child is denoted as $\rho(k, s)$.

In every decision point $j$, a local regret minimizer $\mathfrak{R}_j$ mixes the set of available actions $\mathcal{A}_j$.
In theory, any regret minimization algorithm that operates over the probability simplex can be used as a regret minimizer.
Locally, CFR makes use of regret matching (RM)~\cite{hartandmascolell2000}, which outputs L1-normalized floored regrets as probabilities on each iteration.
Other popular choices include RM\textsuperscript{+}, used by CFR\textsuperscript{+}~\cite{tammelin2014}, discounted RM, used by discounted CFR (DCFR)~\cite{brownandsandholm2019aaai}, and predictive RM (PRM) (respectively, PRM\textsuperscript{+}), used by predictive CFR (PCFR) (respectively, PCFR\textsuperscript{+})~\cite{farinaetal2021}.

Let $\Sigma^+ = \left\{(j, a): \forall j \in \mathcal{J}, a \in \mathcal{A}_j\right\}$ be the set of non-empty sequences and $\Sigma = \Sigma^+ \cup \{\emptyset\}$ be the set of (possibly empty) sequences.
Here, we use $\emptyset$ to denote the empty sequence.
With $p_j$ the parent sequence of a decision point $j$, a sequence-form polytope can be defined as follows:
\[
	\mathcal{X} = \left\{\vec{x} \in \mathbb{R}_{\ge0}^\Sigma: \vec{x}[\emptyset] = 1, \forall j \in \mathcal{J}: \sum_{a \in \mathcal{A}_j} \vec{x}[(j, a)] = \vec{x}[p_j]\right\}.
\]
The CFR family of algorithms is a family of regret minimizers operating over the sequence-form polytope.
On each iteration $t$, CFR constructs a behavior strategy $\vec{b}_j^{(t)} \in \Delta(\mathcal{A}_j)\ \forall j \in \mathcal{J}$ from the outputs of its local regret minimizers, which is then converted into a sequence-form strategy $\vec{x}^{(t)} \in \mathcal{X}$ prior to being returned.
We denote this procedure by $\textsc{NextStrategy}$.
Then, CFR observes utility $\vec{u}^{(t)} \in \mathbb{R}^{\Sigma}$ which is then converted into counterfactual utilities $\vec{u}_j^{(t)} \in \mathbb{R}^{\mathcal{A}_j}\ \forall j \in \mathcal{J}$ to be passed to its local regret minimizers.
Henceforth, we denote this procedure by $\textsc{ObserveUtility}$.
A complete pseudocode description of CFR in sequence-form is provided in Algorithm~\ref{alg:cfr}.
In Line~\ref{line:scalar-division}, we assume that scalar division ($/$) produces an arbitrary strategy when $\left\lVert\left(\vec{r}_j^{(t - 1)}\right)^+\right\rVert_1 = 0$.

The per-iteration time complexity of a player's instance of sequence-form CFR is linear with respect to the size of the player's corresponding TFSDP.
This differs from the per-iteration time complexity of classical CFR~\cite{zinkevichetal2007}: linear with respect to the size of the game tree.
The number of nodes in the game tree is often, but not always, greater than the total number of nodes across its player TFSDPs, so the tree-traversal parts of sequence-form CFR tend to be more efficient than classical CFR.

That said, unlike classical CFR, sequence-form CFR requires an extra step of calculating each player's utility on each iteration, whose worst-case time complexity is linear with respect to the size of the game tree.
However, a sparse matrix-vector multiplication can be carried out in two-player settings for this purpose (or tensor operations in multiplayer settings), which has a small constant factor in practice and can be easily parallelized.
As a result, in practice, sequence-form CFR tends to yield performance improvements over the classical implementation of CFR.

\subsection{Parallelized game solvers}

Unpublished works by \citet{reis2015} and by \citet{rudolf2021} have parallelized CFR on CUDA and found orders of magnitude speedups.
However, in Rudolf's work, threads are assigned to every node, which then moves up the tree; this, in the worst case, results in \textbf{quadratic} work.
This contrasts with the linear time complexity of CFR.
Reis's work has several reproducibility issues (\textit{e.g.}, an inaccessible codebase and code snippets containing syntax errors).
Both implementations require each thread to perform a ``large number of control flow statements''~\cite{reis2015} and \textbf{generalized} kernel instructions.

Instead, we reframe CFR as linear algebra operations in our framework.
Parallelizing linear algebra operations is a well-studied problem in the field of systems, and one can frequently take advantage of optimized opcodes.
As it is platform-independent, our approach is also more general than the aforementioned approaches.
Additionally, unlike Reis's, our implementation is open-source and compatible with games in OpenSpiel~\cite{lanctotetal2020}, a popular software library for game-theoretic primitives.

Another notable game-solving algorithm is the adaptation of the excessive gap technique (EGT)~\cite{nesterov} for solving extensive-form games~\cite{gilpinetal}.
\citet{hodaetal} reduce the memory footprint to a square root of the original, while~\citet{gilpinandsandholm2010} introduce sampling for speedup.
Both parallelize the matrix-vector product during the expected utility calculation to achieve speedups (this is akin to calculating the utilities passed to sequence-form CFR).
Further, \citet{kroeretal} introduce GPU implementations of the EGT, but they exploit the unique topology of poker TFSDPs by only parallelizing across the immediate children of the root node in the TFSDP, which can potentially \textbf{limit} parallelism in general, particularly when the root has few children.

Parallelization was instrumental in several imperfect-information game-solving breakthroughs, including Libratus~\cite{brownandsandholm2018} and Cepheus~\cite{tammelinetal2015} -- superhuman or nearly optimal heads-up Texas hold'em agents.
We first discuss the architecture used in Libratus.
During its blueprint strategy computation, Libratus's workload was distributed (typically) over 1 + 195 compute nodes, depending on the card rollouts.
However, the variant of CFR used by Libratus, Monte Carlo CFR (MCCFR)~\cite{lanctotetal2009}, is a stochastic regret minimizer~\cite{farinaetal2020}.
As it operates under the setting of \textbf{stochastic} regret minimization, its parallelization scheme cannot be readily adapted for use in CFR or its tabular variants.

In contrast, the engineering of Cepheus involved parallelizing CFR\textsuperscript{+}, which is tabular.
During game solving, Cepheus split the game of heads-up limit Texas Hold'em after the second betting round into a trunk and subgames.
Then, the subtrees were assigned to a unique compute node.
On each iteration, the trunk node sent probabilities to the subtree compute nodes, which, in turn, returned values back to the trunk node so it could finish its calculations.
Here, the subtree nodes wait for the trunk node to send probabilities, and the trunk node hangs until the subtree computations are finished.
Our parallelization scheme is different from that of Cepheus primarily in that a) we effectively split the game at every depth, potentially yielding higher parallelism, and b) we offer a domain-independent parallelization framework.
Most importantly, our contribution is \textbf{not mutually exclusive} with Cepheus's architecture: our technique can be augmented to Cepheus for further parallelism.

\section{Parallelizing CFR and its tabular variants}

\begin{figure}[t!]
	\hfill
	\begin{minipage}{0.48\textwidth}
		\footnotesize
		\begin{algorithm}[H]
			\SetInd{0.3em}{0.5em}
			\caption{CFR~\cite{zinkevichetal2007,farinaetal2019}}
			\label{alg:cfr}
			\textbf{Input} a TFSDP.
			\BlankLine
			\textsc{NextStrategy}$()$: \\
			\Indp{
				\Comment{Calculate behavioral strategy}
				\For{each decision point $j \in \mathcal{J}$}{
					$\Delta(\mathcal{A}_j) \ni \vec{b}_j^{(t)} \coloneqq \left(\vec{r}_j^{(t - 1)}\right)^+ / \left\lVert\left(\vec{r}_j^{(t - 1)}\right)^+\right\rVert_1$\; \label{line:scalar-division}
				}
				\Comment{Convert to sequence-form}
				$\vec{x}^{(t)}[\emptyset] \coloneqq 1$\;
				\For{each decision point $j \in \mathcal{J}$, top-down}{
					\For{each action $a \in \mathcal{A}_j$}{
						$\vec{x}^{(t)}[(j, a)] \coloneqq \vec{x}^{(t)}[p_j] \cdot \vec{b}_j^{(t)}[a]$\;
					}
				}
				\Return $\vec{x}^{(t)} \in \mathcal{X}$\;
			}
			\Indm
			\BlankLine
			\textsc{ObserveUtility}$\left(\vec{u}^{(t)} \in \mathbb{R}^{\Sigma}\right)$: \\
			\Indp{
				\Comment{Memoize counterfactual utilities}
				$\vec{v}^{(t)} \coloneqq \vec{0} \in \mathbb{R}^\mathcal{P}$\;
				\For{each node $p \in \mathcal{P}$, bottom-up}{
					\If{$\mathcal{J} \ni j \coloneqq p$}{
						$\vec{v}^{(t)}[j] \coloneqq \sum_{a \in \mathcal{A}_j} b_j^{(t)}[a] \cdot (\vec{u}^{(t)}[(j, a)] + \vec{v}^{(t)}[\rho(j, a)])$\;
					}
					\ElseIf{$\mathcal{K} \ni k \coloneqq p$}{
						$\vec{v}^{(t)}[k] \coloneqq \sum_{s \in \mathcal{S}_k} \vec{v}^{(t)}[\rho(k, s)]$\;
					}
				}
				\Comment{Update counterfactual regrets}
				\For{each decision point $j \in \mathcal{J}$}{
					\For{each action $a \in \mathcal{A}_j$}{
						$\vec{u}_j^{(t)}[a] \coloneqq \vec{u}^{(t)}[(j, a)] + \vec{v}^{(t)}[\rho(j, a)]$\;
					}
					$\mathbb{R}^{\mathcal{A}_j} \ni \vec{r}_j^{(t)} \coloneqq \vec{r}_j^{(t - 1)} + \vec{u}_j^{(t)} - \left(\vec{b}_j^{(t)}\right)^\top \vec{u}_j^{(t)}$\;
				}
			}
			\Indm
		\end{algorithm}
	\end{minipage}
	\hfill
	\begin{minipage}{0.48\textwidth}
		\footnotesize
		\begin{algorithm}[H]
			\SetInd{0.3em}{0.5em}
			\caption{Parallelized CFR (ours)}
			\label{alg:parallelized-cfr}
			\textbf{Input:} A TFSDP.
			\BlankLine
			\textsc{NextStrategy}$()$: \label{line:next-strategy} \\
			\Indp{
				\Comment{Calculate behavioral strategy}
				$\mathbb{R}_{\ge 0}^{\Sigma^+} \ni \vec{b}^{(t)} \coloneqq \left(\vec{r}^{(t - 1)}\right)^+ \oslash \left(\mathbf{C}^\top\left(\mathbf{C}\,\left(\vec{r}^{(t - 1)}\right)^+\right)\right)$\; \label{line:hadamard-division} \label{line:behavioral-strategy-calculation}
				\Comment{Convert to sequence-form}
				$\vec{y}^{(t)} \coloneqq \vec{e}_{p^*} \in \mathbb{R}_{\ge 0}^\mathcal{P}$\; \label{line:unit-vector}
				\For{each level $\mathbf{L}^{(d)}$ with weights $\vec{b}^{(t)}$, top-down}{ \label{line:top-down-begin}
					$\vec{y}^{(t)} \coloneqq \vec{y}^{(t)} + \left(\mathbf{L}^{(d)}\right)^\top \vec{y}^{(t)}$\; \label{line:top-down-end}
				}
				$\mathcal{X} \ni \vec{x}^{(t)} = \mathbf{A}^\top \vec{y}^{(t)}$\; \label{line:strategy-conversion}
				\Return $\vec{x}^{(t)}$\; \label{line:return}
			}
			\Indm
			\BlankLine
			\textsc{ObserveUtility}$\left(\vec{u}^{(t)} \in \mathbb{R}^{\Sigma}\right)$: \label{line:observe-utility} \\
			\Indp{
				\Comment{Memoize counterfactual utilities}
				$\vec{v}^{(t)} \coloneqq \vec{0} \in \mathbb{R}^\mathcal{P}$\; \label{line:zeros}
				$\vec{w}^{(t)} \coloneqq \mathbf{A}\,\vec{u}^{(t)} \in \mathbb{R}^\mathcal{P}$\; \label{line:utility-conversion}
				\For{each level $\mathbf{L}^{(d)}$ with weights $\vec{b}^{(t)}$, bottom-up}{ \label{line:bottom-up-begin}
					$\vec{v}^{(t)} \coloneqq \vec{v}^{(t)} + \mathbf{L}^{(d)}\,\vec{w}^{(t)} + \mathbf{L}^{(d)}\,\vec{v}^{(t)}$\; \label{line:bottom-up-end}
				}
				\Comment{Update counterfactual regrets}
				$\mathbb{R}^{\Sigma^+} \ni \vec{q}^{(t)} \coloneqq \mathbf{B}^\top\left(\vec{w}^{(t)} + \vec{v}^{(t)}\right)$\; \label{line:combination}
				$\mathbb{R}^{\Sigma^+} \ni \vec{r}^{(t)} \coloneqq \vec{r}^{(t - 1)} + \vec{q}^{(t)} - \mathbf{C}^\top\left(\mathbf{C}\left(\vec{b}^{(t)} \odot \vec{q}^{(t)}\right)\right)$\; \label{line:hadamard-product} \label{line:regret-update}
			}
			\Indm
			\vspace{4.55em}
		\end{algorithm}
	\end{minipage}
\end{figure}

We are now ready to present our results.
As previously alluded to, we parallelize CFR by rewriting it in linear algebra, after which we can apply standard parallelization techniques for linear algebra operations.
Later, we demonstrate that our parallelization framework can easily be applied to other mainstream tabular variants of CFR, such as CFR\textsuperscript{+}~\cite{tammelin2014}, DCFR~\cite{brownandsandholm2019aaai}, PCFR, and PCFR\textsuperscript{+}~\cite{farinaetal2021}.
We begin this section with an overview of our technique.

\subsection{Conversion via logic matrices}

Sequence-form CFR, whose pseudocode is shown in Algorithm~\ref{alg:cfr}, involves vectors whose values are associated with either sequences (\textit{i.e.}, in $\mathbb{R}^\Sigma$) or nodes (\textit{i.e.}, in $\mathbb{R}^{\mathcal{P}}$).
For example, at any time $t$, an output sequence-form strategy $\vec{x}^{(t)} \in \mathcal{X} \subset \mathbb{R}^\Sigma_{\ge 0}$ and observed utility $\vec{u}^{(t)} \in \mathbb{R}^\Sigma$ associates each of their element with a sequence $(j, a) \in \Sigma$.
Similarly, in the $\textsc{ObserveUtility}$ function, memoized counterfactual utilities $\vec{v} \in \mathbb{R}^\mathcal{P}$ associate each element with a node $p \in \mathcal{P}$.
As we plan on using these values in linear algebra operations, CFR presents a need to reconcile values in different `representations'.
We do so by converting vectors between $\mathbb{R}^\Sigma$ and $\mathbb{R}^{\mathcal{P}}$ using logic matrices.

Logic matrices are matrices whose entries are in the Boolean domain $\{0, 1\}$.
We show how a logic matrix can be used to convert vectors between a per-sequence representation (\textit{i.e.}, $\mathbb{R}^\Sigma$) and a per-node representation (\textit{i.e.}, $\mathbb{R}^\mathcal{P}$).
By this, we mean to associate with $\rho(j, a)$ a value previously associated with a sequence $(j, a)$ and vice versa.
Note that $\rho(\cdot)$ maps an empty sequence $\emptyset$ to the root node.
We denote this logic matrix by $\mathbf{A} \in \mathbb{R}^{\mathcal{P} \times \Sigma}$ and define each of its element as follows: $A_{p,(j,a)} = \1_{p = \rho(j,a)}$.
Here, $\1_{(\cdot)}$ is an indicator function, which outputs $1$ when the condition in brackets is true or $0$ if otherwise.
Then, we can trivially convert between a vector $\vec{z} \in \mathbb{R}^{\mathcal{P}}$ and its per-sequence representation $\vec{z}\,' \in \mathbb{R}^{\Sigma}$ by calculating $\vec{z} = \mathbf{A}\,\vec{z}\,'$.
We also have $\vec{z}\,' = \mathbf{A}^\top\vec{z}$.
We further remark that $nnz(\mathbf{A}) = |\Sigma|$, where $nnz(\cdot)$ is a function that returns the number of non-zero elements, and each row or column of $\mathbf{A}$ has at most a single non-zero value.
These are later used in our complexity analysis.

Additionally, by definition, a behavioral strategy $\vec{b}^{(t)}$ at any time $t$ can be represented with a vector in $\mathbb{R}^{\Sigma^+}$.
Recall that $\Sigma^+$ is the set of non-empty sequences and contains all pairs of decision points and available actions.
Henceforth, we refer to vectors in $\mathbb{R}^{\Sigma^+}$ as being in the behavioral representation.
Behaviorally-represented vectors have exactly one less element than those in a per-sequence representation: the empty sequence $\emptyset$ is no longer associated with a value.
The behavioral representation is also sufficient to capture cumulative regrets $\vec{r}^{(t)}$.
As done previously, we can define the following logic matrix $\mathbf{B} = \mathbf{A}_{:,2:} \in \mathbb{R}^{\mathcal{P} \times \Sigma^+}$ (assuming the first column in $\mathbf{A}$ corresponds to $\emptyset$).
With this, we can translate from $\vec{z} \in \mathbb{R}^{\mathcal{P}}$ to its behavioral representation $\vec{z}\,'' \in \mathbb{R}^{\Sigma^+}$ by calculating $\vec{z}\,'' = \mathbf{B}^\top\,\vec{z}$.
Note that $nnz(\mathbf{B}) = |\Sigma^+|$ and each row/column of $\mathbf{B}$ has at most a single non-zero.

Next, we show how a logic matrix can be used to combine values associated with available actions in each decision point.
This ability is crucial in calculating behavioral strategies $\vec{b}^{(t)}$ and cumulative regrets $\vec{r}^{(t)}$ at any $t$.
For example, to calculate the behavioral strategy in $\textsc{NextStrategy}$, the floored counterfactual regrets must be L1 normalized action-wise.
This entails obtaining the sums of floored counterfactual regrets associated with available actions in each decision point.
For this purpose, we define $\mathbf{C} \in \mathbb{R}^{\mathcal{J} \times \Sigma^+}$ where $C_{j,(j',a)} = \1_{j=j'}$.
Then, given $\vec{z}\,'' \in \mathbb{R}^{\Sigma^+}$, one can obtain the action-wise sums by calculating $\mathbf{C}\,\vec{z}\,''$.
One can also propagate the collapsed values back with $\mathbf{C}^\top\mathbf{C}\,\vec{z}\,''$.
Then, $\vec{z}\,''$ can be L1-normalized action-wise by calculating $\vec{z}\,'' \oslash \left(\mathbf{C}^\top\mathbf{C}\,\vec{z}\,''\right)$ .
Here, `$\oslash$' denotes the Hadamard division operator, which carries out element-wise division.
Note that $nnz(\mathbf{C}) = |\Sigma^+|$ and each row and column of $\mathbf{C}$ has at most $\max_{j \in \mathcal{J}}{|\mathcal{A}_j|}$ and $1$ non-zeros, respectively.

\subsection{Parallelized tree traversal}

Notably, in sequence-form CFR, both \textsc{NextStrategy} and \textsc{ObserveUtility} require complete tree traversals over the input TFSDP, top-down for the former and bottom-up for the latter.
As is our goal, we convert the tree traversal step in the original algorithm into linear algebra operations.
For this, we take inspiration from GraphBLAS~\cite{kepneretal2016}, whose core insight is that multiplying a graph's adjacency matrix with a vector is akin to a single step in breadth-first search.

TFSDPs can be decomposed into level structures.
Define an adjacency matrix $\mathbf{L}^{(d)} \in \mathbb{R}^{\mathcal{P} \times \mathcal{P}}$ for each depth $d$.
A non-zero element $L^{(d)}_{p_1,p_2}$ denotes that there is an edge of that weight connecting $p_1$ at depth $d - 1$ and $p_2$ at depth $d$.
With these level graphs, a top-down tree traversal can be achieved by applying the level graphs in the following order: $\mathbf{L}^{(1)}, \mathbf{L}^{(2)}, \hdots, \mathbf{L}^{(k)}$ where $k$ is the height of the tree.
Applying them in the reverse order results in a bottom-up traversal of the tree.
The number of non-zeros in $L^{(d)}$ is equal to the number of edges in the corresponding level structure.
Similarly, the maximum number of entries in any row in $L^{(d)}$ is equal to the size of the largest set of nodes in depth $d$ that shares the same parent node.
The maximum number of entries in any column is equal to $1$.

\subsection{Algorithm description}

Parallelized CFR, combining the above insights, is presented as Algorithm~\ref{alg:parallelized-cfr}.
In Line~\ref{line:hadamard-division}, we assume that the Hadamard division operator `$\oslash$' produces an arbitrary strategy for $j$ when $\left\lVert\left(\vec{r}_j^{(t - 1)}\right)^+\right\rVert_1 = 0$.
In Line~\ref{line:hadamard-product}, we use `$\odot$' to denote the Hadamard product operator, which carries out element-wise multiplication.
We ensured that the comments provided inside the pseudocode match those provided in the non-parallelized CFR's pseudocode in Algorithm~\ref{alg:cfr} so they can be compared easily.

The \textsc{NextStrategy} function, which returns a sequence-form strategy for iteration $t$, is defined in Line~\ref{line:next-strategy}.
In Line~\ref{line:behavioral-strategy-calculation}, we L1-normalize the floored counterfactual regrets $\left(\vec{r}^{(t)}\right)^+$ at time $t$ to produce the behavioral strategy $\vec{b}^{(t)}$.
The logic matrix $\mathbf{C}$ is used during this step.

In the following lines, the behavioral strategy $\vec{b}^{(t)}$ is converted to a sequence-form strategy (\textit{i.e.}, in $\mathcal{X}$).
To do this, we first define $\vec{y}^{(t)} = \vec{e}_{p^*} \in \mathbb{R}^{\mathcal{P}}_{\ge 0}$ in Line~\ref{line:unit-vector}.
Here, $p^*$ denotes the root of the input TFSDP, and $\vec{e}_p$ is used to denote a unit vector whose element corresponding to $p \in \mathcal{P}$ is set to one while the rest is set to zero.
Then, a top-down traversal is performed between Lines~\ref{line:top-down-begin} and~\ref{line:top-down-end} to propagate the probabilities from $\vec{b}^{(t)}$ down the tree.
Here, the edge weights in the adjacency matrix of the level graph are defined depending on the type of edge.
When the edge corresponds to applying an action $a$ at a decision point $j$, the weight is set to $\vec{b}^{(t)}_{(j,a)}$.
Otherwise, the edge corresponds to a signal and its weight is set to $1$.
After the tree traversal, $\vec{x}^{(t)} \in \mathcal{X}$ is calculated from $\vec{y}^{(t)}$ in Line~\ref{line:strategy-conversion} by converting it from the per-node representation to the per-sequence representation.
The logic matrix $\mathbf{A}$ is used for this step.
Finally, $\vec{x}^{(t)}$ is returned in Line~\ref{line:return}.

\textsc{ObserveUtility}, which, given utilities, updates the cumulative counterfactual regrets, is in Line~\ref{line:observe-utility}.
In Line~\ref{line:zeros}, $\vec{v}^{(t)}$ is initialized with zeros.
Then, in Line~\ref{line:utility-conversion}, the input utility $\vec{u}^{(t)}$ is converted from a per-sequence representation to a per-node representation, and is denoted as $\vec{w}^{(t)}$.
The logic matrix $\mathbf{A}$ is used during this step.
From Line~\ref{line:bottom-up-begin} to Line~\ref{line:bottom-up-end}, a bottom-up tree traversal is carried out to accumulate the counterfactual utilities $\vec{u}^{(t)}$ up the tree.
The accumulated values are stored in $\vec{v}^{(t)}$.

Next, $\vec{v}^{(t)}$ is used to update the counterfactual regrets.
Line~\ref{line:combination} combines $\vec{v}^{(t)}$ and $\vec{w}^{(t)}$, converts the result from a per-node representation to a behavioral representation using $\mathbf{B}$, and denotes it as $\vec{q}^{(t)}$.
It is then used in Line~\ref{line:regret-update} to update the counterfactual regrets.
Updating the cumulative counterfactual regrets entails adding the counterfactual values, which, in turn, requires calculating the expected values at each decision point.
These are action-wise sums of the counterfactual values $\vec{q}^{(t)}$, weighted by the probabilities of applying the corresponding actions, given by the behavioral strategy $\vec{b}^{(t)}$.
We can collapse the Hadamard product $\vec{b}^{(t)} \odot \vec{q}^{(t)}$ using $\mathbf{C}$ to obtain the expected values at each decision point.
Just as we did in Line~\ref{line:hadamard-division}, the collapsed values are propagated back to obtain the counterfactual values $\vec{q}^{(t)} - \mathbf{C}^\top\left(\mathbf{C}\left(\vec{b}^{(t)} \odot \vec{q}^{(t)}\right)\right)$, which is then added to the counterfactual regrets.

\subsection{Complexity analysis}
\label{sec:complexity-analysis}

We use the \textit{work-depth model} in our complexity analysis.
The work ($W$) refers to the total number of operations performed by the algorithm, while the depth ($D$) refers to the length of the longest span of operations in sequence.
We use \textit{compressed sparse row (CSR)} sparse matrix operations in our implementations.
Sparse matrix-vector multiplication has total work linear in the number of non-zeros in the matrix and depth proportional to the maximum number of non-zeros in a row~\cite{blelloch}.
All other linear algebra operations used in Algorithm~\ref{alg:parallelized-cfr} -- vector addition, vector subtraction, the Hadamard division, and the Hadamard product -- have linear work and constant depth.
With this, we have the following lemma for the complexity of the \textsc{NextStrategy} function.

\begin{lemma}
	\label{lmm:next-strategy-complexity}
	The \textsc{NextStrategy} function in Algorithm~\ref{alg:parallelized-cfr} has the total work of $W = \Theta(|\mathcal{P}|)$ and the depth of $D = \Theta(\log{(\max_{j \in \mathcal{J}}{|\mathcal{A}_j|})} + k)$.
\end{lemma}
Recall that $k$ is the height of the TFSDP.
Due to space constraints, the proof is relegated to Appendix~\ref{sec:next-strategy-complexity}.
Now, we provide the complexity of the \textsc{ObserveUtility} function in Lemma~\ref{lmm:observe-utility-complexity}.
Defining $B$ as the degree of the TFSDP, \textit{i.e.}, the maximum number of children of any node,

\begin{lemma}
	\label{lmm:observe-utility-complexity}
	The \textsc{ObserveUtility} function in Algorithm~\ref{alg:parallelized-cfr} has the total work of $W = \Theta(|\mathcal{P}|)$ and the depth of $D = \Theta(k\log{B})$.
\end{lemma}
The proof is in Appendix~\ref{sec:observe-utility-complexity}. Lemmas~\ref{lmm:next-strategy-complexity} and~\ref{lmm:observe-utility-complexity} can be combined to give the total work and depth of parallelized CFR, described in Algorithm~\ref{alg:parallelized-cfr}.

\begin{theorem}
	\label{thm:complexity}
	Each iteration of parallelized sequence-form CFR, described in Algorithm~\ref{alg:parallelized-cfr}, has the total work of $W = \Theta(|\mathcal{P}|)$ and the depth of $D = \Theta(k\log{B})$.
\end{theorem}
The proof is in Appendix~\ref{sec:complexity}.

\begin{remark}
	Recall that the per-iteration time complexity of non-parallelized sequence-form CFR is linear with respect to the number of nodes in the TFSDP, \textit{i.e.}, $O(|\mathcal{P}|)$.
	Since $W$ is also proportional to the size of the TFSDP, no efficiency is lost even when there is no parallelism.
\end{remark}
Combining $W$ and $D$, we obtain the following theoretical speedup potential:

\begin{corollary}
	\label{cll:parallelism}
	The parallelism of each iteration of parallelized sequence-form CFR, described in Algorithm~\ref{alg:parallelized-cfr}, is $\frac{|\mathcal{P}|}{k\log{B}}$.
\end{corollary}
The proof is in Appendix~\ref{sec:parallelism}.

\begin{remark}
	Intuitively, as the height of the TFSDP $k$ increases, parallelism decreases.
	What seems counterintuitive at first glance is that the increase in $B$, the branching factor, also decreases parallelism, whereas we expect the opposite to be true.
	This is the consequence of Lemma~\ref{lmm:observe-utility-complexity}, which introduces the $\log{B}$ factor due to the depth of the dot product operation in the CSR sparse matrix-vector multiplication.
	Still, it is worth noting that, in practice, increasing $\log B$ decreases the height $k$, assuming the size of the TFSDP stays constant.
\end{remark}
We conclude this section with the space complexity of our parallelized CFR.

\begin{theorem}
	\label{thm:space-complexity}
	The space complexity of parallelized sequence-form CFR, described in Algorithm~\ref{alg:parallelized-cfr}, is linear with respect to the size of the TFSDP, \textit{i.e.}, $\Theta(|\mathcal{P}|)$.
\end{theorem}
The proof is in Appendix~\ref{sec:space-complexity}.
It follows from the fact that the space complexity of sparse matrix-vector multiplication is proportional to the number of non-zeros, and every vector operation in Algorithm~\ref{alg:parallelized-cfr} has linear space complexity.

\begin{remark}
	The space complexity is identical to that of unparallelized sequence-form CFR, which is also $\Theta(|\mathcal{P}|)$.
	But this differs from that of the classical implementation of CFR, which is $\Theta(|\Sigma^+|)$.
	However, note that $|\Sigma^+| = O(|\mathcal{P}|)$.
\end{remark}

\subsection{Parallelizing the tabular variants of CFR}

In this section, we demonstrate that our parallelization framework is powerful enough to be applied to modern, mainstream members of the tabular family of CFR variants, including the state-of-the-art.

\subsubsection{Alternating player updates}

The classical implementation of CFR~\cite{zinkevichetal2007} interleaves the player updates so their strategies are updated simultaneously in each iteration.
Later works found that alternating the player updates, that is, updating the strategy of one player before the other and so on, empirically exhibits much faster convergence.
Alternating the player updates is trivial in the sequence-form implementation of CFR: one simply needs to avoid interleaving the calls to both players' \textsc{NextStrategy} and \textsc{ObserveUtility} procedures.
The same can be applied to parallelized sequence-form CFR.

\begin{wrapfigure}{R}{0.5\linewidth}
	\vspace{-1em}
	\hfill
	\begin{minipage}{0.95\linewidth}
		\begin{algorithm}[H]
			\caption{Excerpt of parallelized CFR\textsuperscript{+}}
			\label{alg:parallelized-cfr+}
			\textsc{ObserveUtility\textsuperscript{*}}$\left(\vec{u}^{(t)} \in \mathbb{R}^{\Sigma}\right)$: \\
			\Indp{
				\Comment{\textsc{ObserveUtility}\:from\:parallel\:CFR}
				\textsc{ObserveUtility}$(\vec{u}^{(t)})$\;
				$\mathbb{R}^{\Sigma^+} \ni \vec{r}^{(t)} \coloneqq \left(\vec{r}^{(t)}\right)^+$\;
			}
			\Indm
		\end{algorithm}
	\end{minipage}
\end{wrapfigure}

\subsubsection{CFR\textsuperscript{+}}

\textit{CFR\textsuperscript{+}}~\cite{tammelin2014}, unlike CFR, floors the counterfactual regrets at each iteration and involves alternating player updates.
This ensures that unpromising actions are still played with zero probability, but can be quickly explored when they become promising.
Our parallelization framework can be extended to CFR\textsuperscript{+} simply by appending a single statement at the end of \textsc{ObserveUtility}, as shown in Algorithm~\ref{alg:parallelized-cfr+}.

\begin{wrapfigure}{R}{0.5\linewidth}
	\vspace{-2em}
	\hfill
	\begin{minipage}{0.95\linewidth}
		\begin{algorithm}[H]
			\caption{Excerpt of parallelized DCFR}
			\label{alg:parallelized-dcfr}
			\textsc{ObserveUtility\textsuperscript{*}}$\left(\vec{u}^{(t)} \in \mathbb{R}^{\Sigma}\right)$: \\
			\Indp{
				\Comment{\textsc{ObserveUtility}\:from\:parallel\:CFR}
				\textsc{ObserveUtility}$(\vec{u}^{(t)})$\;
				$\mathbb{R}^{\Sigma^+} \ni \vec{r}^{(t)} \coloneqq \vec{r}^{(t)} \odot \left(\begin{cases} \frac{t^\alpha}{t^\alpha + 1} & r^{(t)}_{(j,a)} > 0 \\ \frac{t^\beta}{t^\beta + 1} & r^{(t)}_{(j,a)} < 0 \\ 1 & r^{(t)}_{(j,a)} = 0 \\ \end{cases}\right)_{(j,a) \in \Sigma^+}$\;
			}
			\Indm
		\end{algorithm}
	\end{minipage}
\end{wrapfigure}

\subsubsection{Discounted CFR (DCFR)}

\textit{Discounted CFR (DCFR)}~\cite{brownandsandholm2019aaai}, unlike CFR, discounts the cumulative counterfactual regrets, depending on their signs.
On each iteration, positive regrets are multiplied by $t^\alpha/(t^\alpha+1)$ whereas negative regrets are multiplied by $t^\beta/(t^\beta+1)$.
CFR\textsuperscript{+} is a special case of DCFR where $\alpha = \infty$ and $\beta = -\infty$.
\citet{brownandsandholm2019aaai} remarked that setting $(\alpha, \beta) = (1.5, 0)$, quadratically weighting each iteration's contributions to the average strategy, and alternating player updates consistently outperforms CFR\textsuperscript{+} in practice.
Algorithm~\ref{alg:parallelized-dcfr} demonstrates how our framework can be applied to parallelize DCFR.

\begin{wrapfigure}{R}{0.5\linewidth}
	\vspace{-4em}
	\hfill
	\begin{minipage}{0.95\linewidth}
		\begin{algorithm}[H]
			\caption{Excerpt of PCFR\textsuperscript{(+)}}
			\label{alg:predictive-cfr}
			\textsc{NextStrategy\textsuperscript{*}}$\left(\vec{m}^{(t)} \in \mathbb{R}^{\Sigma}\right)$: \\
			\Indp{
				\Comment{Back\:up\:counterfactual\:regrets}
				$\mathbb{R}^{\Sigma^+} \ni \vec{\rho}^{(t)} \coloneqq \vec{r}^{(t)}$\;
				\Comment{Copy\:prev.\:behavioral\:strategy}
				$\mathbb{R}^{\Sigma^+}_{\ge 0} \ni \vec{b}^{(t)} \coloneqq \vec{b}^{(t-1)}$\;
				\Comment{Observe prediction}
				\textsc{ObserveUtility}$(\vec{m}^{(t)})$\;
				\Comment{\textsc{NextStrategy}\:from\:parallel\:CFR}
				$\mathcal{X} \ni \vec{x}^{(t)} \coloneqq$ \textsc{NextStrategy}$()$\;
				\Comment{Restore\:counterfactual\:regrets}
				$\mathbb{R}^{\Sigma^+} \ni \vec{r}^{(t)} \coloneqq \vec{\rho}^{(t)}$\;
				\Return $\vec{x}^{(t)}$\;
			}
			\Indm
		\end{algorithm}
	\end{minipage}
\end{wrapfigure}

\subsubsection{Predictive CFR (PCFR) and PCFR\textsuperscript{+}}

\textit{Predictive CFR (PCFR)} (respectively, PCFR\textsuperscript{+})~\cite{farinaetal2021} is the predictive counterpart of CFR (respectively, CFR\textsuperscript{+}).
In both PCFR and PCFR\textsuperscript{+}, a prediction vector $\vec{m}^{(t)}$ is used to guess the next iteration's instantaneous counterfactual regrets.
A popular option is setting the prediction vector to the previous iteration's counterfactual utilities, \textit{i.e.}, $\vec{m}^{(t)} = \vec{u}^{(t - 1)}$, which has been shown to perform quite well in practice.
Empirically, PCFR\textsuperscript{+} tends to exhibit state-of-the-art performance in most applications.
We demonstrate how our framework can be applied to both PCFR and PCFR\textsuperscript{+} in Algorithm~\ref{alg:predictive-cfr}.

\section{Experiments}
\label{sec:experiments}

\begin{figure}[t!]
	\centering
	\includegraphics[width=\linewidth]{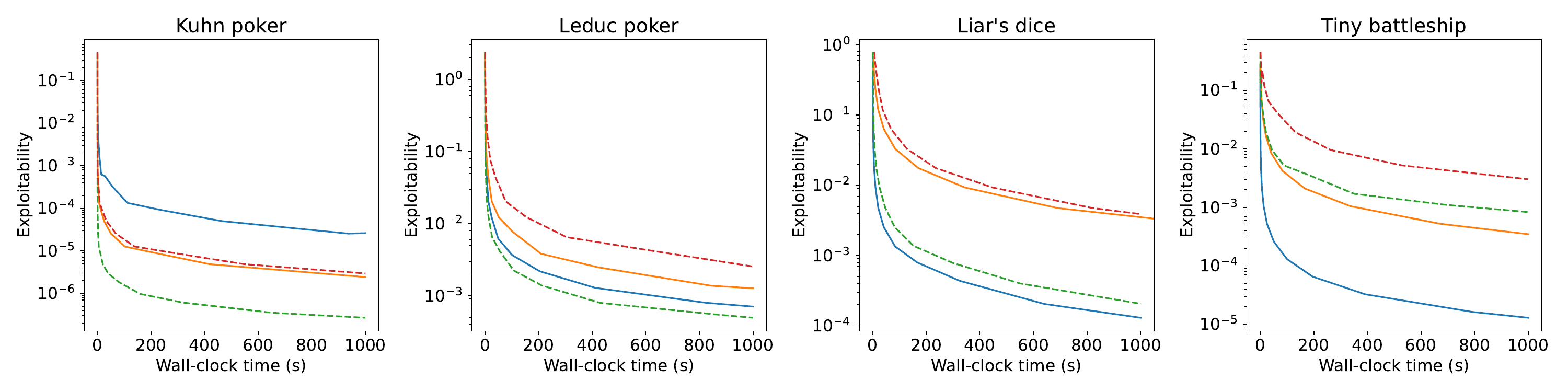}
	\includegraphics[width=0.87\linewidth]{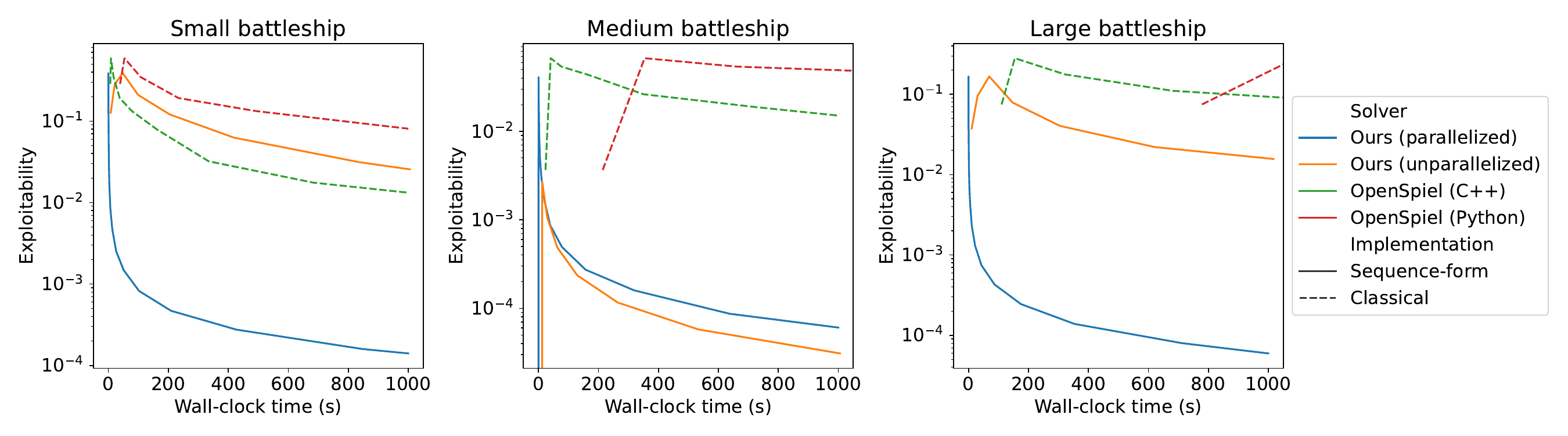}
	\caption{Exploitabilities versus wall-clock time for our CFR implementations and those of OpenSpiel.}
	\label{fig:gpugt:plot}
\end{figure}

For our experiments, we created two CFR implementations: one with parallelism and another without, and tested their runtime performance.
In the parallelized version, we used CuPy~\cite{cupy_learningsys2017} for GPU-accelerated linear algebra operations.
While there is a lack of good benchmarks for CFR performance, OpenSpiel~\cite{lanctotetal2020} offers popular Python and C++ CFR implementations, which we use as baselines to compare against.
We justify the choice of our baselines later in the discussions.
We used double-precision floating-point numbers in our implementations for fair comparison, as OpenSpiel uses the aforesaid format.
Our testbench specification is given in Appendix~\ref{sec:resources}.

\begin{wraptable}{R}{0.5\linewidth}
	\centering
	\caption{Speedups (or slowdowns) achieved by our parallelized implementation compared to OpenSpiel's C++ implementation.}
	\label{tab:condensed}
	\begin{tabular}{lcc}
		\toprule
		Game & \# nodes & Speedup \\
		\midrule
		Kuhn poker & 58 & $-3.6 \times 10^2$ \\
		Leduc poker & $9.4 \times 10^3$ & $-1.9$ \\
		Tiny battleship & $5.9 \times 10^4$ & $6.9$ \\
		Liar's dice & $2.9 \times 10^5$ & $1.5$ \\
		Small battleship & $2.3 \times 10^6$ & $2.1 \times 10^2$ \\
		Medium battleship & $1.5 \times 10^7$ & $1.2 \times 10^3$ \\
		Large battleship & $5.2 \times 10^7$ & $4.1 \times 10^3$ \\
		\bottomrule
	\end{tabular}
\end{wraptable}

We conducted our experiments on seven common benchmark games provided by OpenSpiel~\cite{lanctotetal2020}: Kuhn poker, Leduc poker, liar's dice, and four battleship games~\cite{farinaetal} of varying sizes, whose initialization parameters are shown in Appendix~\ref{sec:battleship}.
For every implementation and game, CFR was run for 1,000 seconds.
We also ensured that at least eight iterations of CFR were executed for each setting.
Note that the minimum number of iterations can be low because our goal is not to solve these games but instead to observe the runtime behavior of different implementations of CFR.
The number of nodes in the game trees of the games tested ranges from 58 to $6.7 \times 10^7$ and is shown in Table~\ref{tab:condensed}.

For each run, we recorded the iteration times, wall-clock times, and exploitabilities.
Exploitability is a standard metric for quantifying the quality of a strategy profile in imperfect-information games (lower is better).
The computed exploitabilities over the wall-clock time are plotted in Figure~\ref{fig:gpugt:plot}.
Note that some lines seem horizontally shifted because the corresponding implementations took substantially more time to produce even a single iterate.
Also, the speedups achieved by our parallelized implementation on a GPU compared to OpenSpiel's more performant C++ implementation on a CPU are shown in Table~\ref{tab:condensed}.
Here, positive (respectively, negative) values denote how many times faster (respectively, slower) our parallelized implementation is compared to OpenSpiel's C++ implementation.

On smaller games, OpenSpiel's implementations outperform our parallelized implementation.
Here, it is clear that the GPU overhead makes parallelization impractical on smaller scales.
However, as the game size grows, the speed of our parallelized implementation overtakes the others', and the resulting speedup becomes more significant.
This indicates that, in practice, a high degree of parallelization is achieved by our technique.
Overall, our parallelized implementation is up to about $2.7 \times 10^4$ and $4.1 \times 10^3$ times faster than OpenSpiel's Python and C++ implementations, respectively.
We provide additional results, including iteration times and memory usage, in Appendix~\ref{sec:additional}.

\section{Discussion}

Although parallelizing CFR does not necessarily increase the size of the game that can be solved, it can make finding an optimal solution to these games much faster.
One prime use case of our parallelization framework for researchers is in developing new variants of CFR and benchmarking them.
In these settings, games are typically at most modestly large, and from thousands to tens of thousands of CFR iterations are often executed.
Additionally, our framework can be useful in solving games that are played in a repeated setting, whose individual game trees usually differ by little to none across repetitions.
An example of this is poker, whose action depth is determined by stack sizes.

While we acknowledge that comparing the performance of our parallelized implementation on a GPU to that of OpenSpiel on a CPU may not facilitate a comparison that is entirely fair, our experiments nonetheless showed that parallelism can achieve extreme speedups, even in more regular use cases.
Also worth noting is that our implementation achieves this performance while completely retaining compatibility with games defined in OpenSpiel, which, again, is a widely used software library for game-theoretic primitives.
This fact alone could be of separate interest to the extensive-form game-solving community.

The degree of parallelism that can be achieved by our implementation also depends on the hardware being used (\textit{e.g.}, the number of CUDA cores on a GPU, the number of CPU threads, \textit{etc.}).
While CPU multi-threading can technically be used in place of GPUs to parallelize CFR, we would expect this to be less effective, especially since GPUs have optimized opcodes for performing large-scale linear algebra operations, unlike most CPUs.

While our focus was on parallelizing sequence-form CFR, our ideas can also be applied to parallelize classical CFR, although several differences arise.
Most crucially, the space complexity of parallelized classical CFR would be linear with respect to the number of game tree nodes, which is generally larger than the space complexity of parallelized sequence-form CFR (linear with respect to the size of the tree-form sequential decision processes).
Indeed, parallelized classical CFR requires some of the values specific to each player (e.g., player reach probabilities) to be needlessly duplicated across nodes, whereas this can be handled more efficiently in sequence-form CFR.
Additionally, parallelized classical CFR requires introducing an additional logic matrix that represents, for each node, the player whose turn it is to act.
This logic matrix would then be used to select the relevant values for each player during the parallelized tree traversal.
To summarize, while our parallelization framework can also be applied to classical CFR, doing so results in an implementation that is less memory efficient than what can be achieved by parallelizing sequence-form CFR.
Again, note that classical CFR and sequence-form CFR are equivalent in the sense that they produce identical iterates, so by parallelizing sequence-form CFR, we have also indirectly parallelized classical CFR.

\section{Conclusions and future research}

In this paper, we provided a generalized framework for parallelizing the tabular members of the CFR family of algorithms, and gave algorithmic descriptions of the parallelized counterparts of CFR, CFR\textsuperscript{+}, DCFR, PCFR, and PCFR\textsuperscript{+}.
In our analysis using the work-depth model, the total work of the resulting parallel algorithm is equivalent to the time complexity of sequence-form CFR, and the depth scales with the height of the TFSDP.
Our benchmarks show that our implementation on a GPU offers speedups of up to four orders of magnitude compared to Google DeepMind OpenSpiel's~\cite{lanctotetal2020} CFR implementations on a CPU.
While our technique does not necessarily increase the size of the game that can be solved, it can make solving games dramatically faster.
Our technique can be useful in researching new CFR variants and solving games that are played in a repeated setting.

Possible topics for future work include parallelizing other game-theoretic algorithms such as pruned CFR~\cite{brownandssandholm,brownetal2017,brownandsandholm2017}, finding best response, and exploitability calculations.

\section*{Acknowledgements}

This work has been supported by the Vannevar Bush Faculty Fellowship ONR N00014-23-1-2876, National Science Foundation grant RI-2312342, and NIH award A240108S001.
Any opinions, findings, and conclusions or recommendations expressed in this material are those of the authors and do not necessarily reflect the views of the funding agencies.

\bibliographystyle{abbrvnat}
\bibliography{neurips_2026}

@inproceedings{
	zinkevichetal2007,
	author = {Zinkevich, Martin and Johanson, Michael and Bowling, Michael and Piccione, Carmelo},
	booktitle = {Proceedings of the Annual Conference on Neural Information Processing Systems (NeurIPS)},
	pages = {},
	title = {Regret Minimization in Games with Incomplete Information},
	year = {2007},
}

@misc{
	tammelin2014,
	title = {Solving Large Imperfect Information Games Using {CFR+}}, 
	author = {Oskari Tammelin},
	year = {2014},
	eprint = {1407.5042},
	archivePrefix = {arXiv},
	primaryClass = {cs.GT},
}

@inproceedings{
	brownandsandholm2019aaai,
	author = {Brown, Noam and Sandholm, Tuomas},
	title = {Solving imperfect-information games via discounted regret minimization},
	year = {2019},
	booktitle = {Proceedings of the AAAI Conference on Artificial Intelligence (AAAI)},
}

@inproceedings{
	xuetal2024,
	title={Dynamic Discounted Counterfactual Regret Minimization},
	author={Hang Xu and Kai Li and Haobo Fu and Qiang Fu and Junliang Xing and Jian Cheng},
	booktitle={Proceedings of the International Conference on Learning Representations (ICLR)},
	year={2024},
}

@misc{
	lanctotetal2020,
	title = {{O}pen{S}piel: A Framework for Reinforcement Learning in Games}, 
	author = {Marc Lanctot and Edward Lockhart and Jean-Baptiste Lespiau and Vinicius Zambaldi and Satyaki Upadhyay and Julien Pérolat and Sriram Srinivasan and Finbarr Timbers and Karl Tuyls and Shayegan Omidshafiei and Daniel Hennes and Dustin Morrill and Paul Muller and Timo Ewalds and Ryan Faulkner and János Kramár and Bart De Vylder and Brennan Saeta and James Bradbury and David Ding and Sebastian Borgeaud and Matthew Lai and Julian Schrittwieser and Thomas Anthony and Edward Hughes and Ivo Danihelka and Jonah Ryan-Davis},
	year = {2019},
	eprint = {1908.09453},
	archivePrefix = {arXiv},
	primaryClass = {cs.LG},
}

@article{
	hartandmascolell2000,
	author = {Sergiu Hart and Andreu Mas-Colell},
	journal = {Econometrica},
	number = {5},
	pages = {1127--1150},
	publisher = {[Wiley, Econometric Society]},
	title = {A Simple Adaptive Procedure Leading to Correlated Equilibrium},
	volume = {68},
	year = {2000},
}

@inproceedings{cupy_learningsys2017,
  author       = "Okuta, Ryosuke and Unno, Yuya and Nishino, Daisuke and Hido, Shohei and Loomis, Crissman",
  title        = {{CuPy}: A {NumPy}-Compatible Library for {NVIDIA} {GPU} Calculations},
  booktitle    = "Proceedings of the Workshop on Machine Learning Systems (LearningSys) in the Annual Conference on Neural Information Processing Systems (NeurIPS)",
  year         = "2017",
}

@inproceedings{
	lanctotetal2009,
	author = {Lanctot, Marc and Waugh, Kevin and Zinkevich, Martin and Bowling, Michael},
	booktitle = {Proceedings of the Annual Conference on Neural Information Processing Systems (NeurIPS)},
	pages = {},
	title = {{M}onte {C}arlo Sampling for Regret Minimization in Extensive Games},
	year = {2009},
}

@inproceedings{
	tammelinetal2015,
	author = {Tammelin, Oskari and Burch, Neil and Johanson, Michael and Bowling, Michael},
	title = {Solving heads-up limit {T}exas {H}old'em},
	year = {2015},
	booktitle = {Proceedings of the International Joint Conference on Artificial Intelligence (IJCAI)},
	numpages = {8},
	location = {Buenos Aires, Argentina},
}

@article{
	brownandsandholm2018,
	author = {Noam Brown  and Tuomas Sandholm },
	title = {Superhuman {AI} for heads-up no-limit poker: {Libratus} beats top professionals},
	journal = {Science},
	volume = {359},
	number = {6374},
	pages = {418--424},
	year = {2018},
}

@article{
	brownandsandholm2019,
	author = {Noam Brown  and Tuomas Sandholm },
	title = {Superhuman {AI} for multiplayer poker},
	journal = {Science},
	volume = {365},
	number = {6456},
	pages = {885--890},
	year = {2019},
}

@article{
	moravciketal2017,
	author = {Matej Moravčík  and Martin Schmid  and Neil Burch  and Viliam Lisý  and Dustin Morrill  and Nolan Bard  and Trevor Davis  and Kevin Waugh  and Michael Johanson  and Michael Bowling },
	title = {{DeepStack}: Expert-level artificial intelligence in heads-up no-limit poker},
	journal = {Science},
	volume = {356},
	number = {6337},
	pages = {508--513},
	year = {2017},
}

@mastersthesis{reis2015,
  title = {A {GPU} implementation of Counterfactual Regret Minimization},
  author = {Reis, Jo{\~a}o},
  year = {2015},
  school = {University of Porto},
}

@misc{
	rudolf2021,
	author = {Jan Rudolf},
	title = {Counterfactual Regret Minimization on {GPU}},
	year = {2021},
	school = {Czech Technical University in Prague},
}

@inproceedings{
	kepneretal2016,
	author = {Kepner, Jeremy and Aaltonen, Peter and Bader, David and Buluç, Aydin and Franchetti, Franz and Gilbert, John and Hutchison, Dylan and Kumar, Manoj and Lumsdaine, Andrew and Meyerhenke, Henning and McMillan, Scott and Yang, Carl and Owens, John D. and Zalewski, Marcin and Mattson, Timothy and Moreira, Jose},
	booktitle = {Proceedings of the IEEE High Performance Extreme Computing Conference (HPEC)}, 
	title = {Mathematical foundations of the {GraphBLAS}}, 
	year = {2016},
	volume = {},
	number = {},
}

@inproceedings{brownetal2019,
  title = 	 {Deep Counterfactual Regret Minimization},
  author =       {Brown, Noam and Lerer, Adam and Gross, Sam and Sandholm, Tuomas},
  booktitle = 	 {Proceedings of the International Conference on Machine Learning (ICML)},
  year = 	 {2019}
}

@inproceedings{
mcaleeretal2023,
title={{ESCHER}: Eschewing Importance Sampling in Games by Computing a History Value Function to Estimate Regret},
author={Stephen Marcus McAleer and Gabriele Farina and Marc Lanctot and Tuomas Sandholm},
	booktitle={Proceedings of the International Conference on Learning Representations (ICLR)},
year={2023},
}

@inproceedings{farinaetal2019,
  title = 	 {Regret Circuits: Composability of Regret Minimizers},
  author =       {Farina, Gabriele and Kroer, Christian and Sandholm, Tuomas},
  booktitle = 	 {Proceedings of the International Conference on Machine Learning (ICML)},
  year = 	 {2019}
}

@inproceedings{
	farinaetal2021,
	title = {Faster Game Solving via Predictive {B}lackwell Approachability: Connecting Regret Matching and Mirror Descent},
	booktitle = {Proceedings of the AAAI Conference on Artificial Intelligence (AAAI)},
	author = {Farina, Gabriele and Kroer, Christian and Sandholm, Tuomas},
	year = {2021},
}

@inproceedings{farinaetal2020,
author = {Farina, Gabriele and Kroer, Christian and Sandholm, Tuomas},
title = {Stochastic regret minimization in extensive-form games},
year = {2020},
  booktitle = 	 {Proceedings of the International Conference on Machine Learning (ICML)},
articleno = {283},
numpages = {11},
}

@inproceedings{kroeretal,
 author = {Kroer, Christian and Farina, Gabriele and Sandholm, Tuomas},
	booktitle = {Proceedings of the Annual Conference on Neural Information Processing Systems (NeurIPS)},
 pages = {},
 title = {Solving Large Sequential Games with the Excessive Gap Technique},
 year = {2018}
}

@article{kolleretal,
title = {Efficient Computation of Equilibria for Extensive Two-Person Games},
journal = {Games and Economic Behavior},
volume = {14},
number = {2},
pages = {247--259},
year = {1996},
author = {Daphne Koller and Nimrod Megiddo and Bernhard {von Stengel}},
}

@article{vonstengel,
title = {Efficient Computation of Behavior Strategies},
journal = {Games and Economic Behavior},
volume = {14},
number = {2},
pages = {220--246},
year = {1996},
author = {Bernhard {von Stengel}},
}

@inproceedings{brownandssandholm,
 author = {Brown, Noam and Sandholm, Tuomas},
	booktitle = {Proceedings of the Annual Conference on Neural Information Processing Systems (NeurIPS)},
 title = {Regret-Based Pruning in Extensive-Form Games},
 year = {2015}
}

@inproceedings{brownetal2017, title={Dynamic Thresholding and Pruning for Regret Minimization}, booktitle={Proceedings of the AAAI Conference on Artificial Intelligence (AAAI)}, author={Brown, Noam and Kroer, Christian and Sandholm, Tuomas}, year={2017}}

@article{blelloch,
author = {Blelloch, Guy E.},
title = {Programming parallel algorithms},
year = {1996},
publisher = {Association for Computing Machinery},
volume = {39},
number = {3},
journal = {Commun. ACM},
pages = {85--97},
}

@inproceedings{brownandsandholm2017,
author = {Brown, Noam and Sandholm, Tuomas},
title = {Reduced space and faster convergence in imperfect-information games via Pruning},
year = {2017},
  booktitle = 	 {Proceedings of the International Conference on Machine Learning (ICML)},
}

@article{nesterov,
author = {Nesterov, Yuri},
title = {Excessive Gap Technique in Nonsmooth Convex Minimization},
journal = {SIAM Journal on Optimization},
volume = {16},
number = {1},
pages = {235--249},
year = {2005},
}

@inproceedings{gilpinetal,
author="Gilpin, Andrew
and Hoda, Samid
and Pe{\~{n}}a, Javier
and Sandholm, Tuomas",
title="Gradient-Based Algorithms for Finding {N}ash Equilibria in Extensive Form Games",
booktitle="Proceedings of the International Workshop on Internet and Network Economics (WINE)",
year="2007",
}

@article{hodaetal,
author = {Hoda, Samid and Gilpin, Andrew and Pe\`{n}a, Javier and Sandholm, Tuomas},
title = {Smoothing Techniques for Computing {N}ash Equilibria of Sequential Games},
year = {2010},
publisher = {INFORMS},
volume = {35},
number = {2},
journal = {Math. Oper. Res.},
pages = {494--512},
}

@inproceedings{gilpinandsandholm2010,
author = {Gilpin, Andrew and Sandholm, Tuomas},
title = {Speeding up gradient-based algorithms for sequential games},
year = {2010},
booktitle = {Proceedings of the International Conference on Autonomous Agents and Multiagent Systems (AAMAS)},
}

@inproceedings{deaneteal,
 author = {Dean, Jeffrey and Corrado, Greg and Monga, Rajat and Chen, Kai and Devin, Matthieu and Mao, Mark and Ranzato, Marc\textquotesingle aurelio and Senior, Andrew and Tucker, Paul and Yang, Ke and Le, Quoc and Ng, Andrew},
 booktitle = {Proceedings of the Annual Conference on Neural Information Processing Systems (NeurIPS)},
 title = {Large Scale Distributed Deep Networks},
 year = {2012}
}

@misc{goyaletal,
      title={Accurate, Large Minibatch {SGD}: Training {I}mage{N}et in 1 Hour}, 
      author={Priya Goyal and Piotr Dollár and Ross Girshick and Pieter Noordhuis and Lukasz Wesolowski and Aapo Kyrola and Andrew Tulloch and Yangqing Jia and Kaiming He},
      year={2018},
}

@misc{sergeevanddelbalso,
      title={{H}orovod: fast and easy distributed deep learning in {T}ensor{F}low}, 
      author={Alexander Sergeev and Mike Del Balso},
      year={2018},
}

@inproceedings{farinaetal,
 author = {Farina, Gabriele and Ling, Chun Kai and Fang, Fei and Sandholm, Tuomas},
 booktitle = {Proceedings of the Annual Conference on Neural Information Processing Systems (NeurIPS)},
 title = {Correlation in Extensive-Form Games: Saddle-Point Formulation and Benchmarks},
 year = {2019}
}

@inproceedings{zhangetal2026a,
      title={Faster Game Solving via Hyperparameter Schedules}, 
      author={Naifeng Zhang and Stephen McAleer and Tuomas Sandholm},
      year={2026},
	booktitle = {Proceedings of the AAAI Conference on Artificial Intelligence (AAAI)},
}

@misc{zhangetal2026b,
      title={Scale-Invariant Regret Matching and Online Learning with Optimal Convergence: Bridging Theory and Practice in Zero-Sum Games}, 
      author={Brian Hu Zhang and Ioannis Anagnostides and Tuomas Sandholm},
      year={2026},
}


\newpage
\appendix

\section{Omitted proofs in Section~\ref{sec:complexity-analysis}}
\label{sec:omitted-proofs}

This appendix section contains proofs excluded from the main paper content due to space constraints.

\subsection{Proof of Lemma~\ref{lmm:next-strategy-complexity}}
\label{sec:next-strategy-complexity}

\begin{proof}
	We begin by analyzing the total work.
	Line~\ref{line:behavioral-strategy-calculation} features two sparse matrix-vector multiplications and the Hadamard division operation.
	The total work of the sparse matrix-vector multiplications is $\Theta(nnz(\mathbf{C})) = \Theta(|\Sigma^+|)$, whereas the Hadamard division has the work of $\Theta(|\Sigma^+|)$.
	Thus, the total work of Line~\ref{line:behavioral-strategy-calculation} is $\Theta(|\Sigma^+|)$.
	Line~\ref{line:unit-vector} has the work of $\Theta(|\mathcal{P}|)$.
	Next, we analyze the total work of the top-down tree traversal between Lines~\ref{line:top-down-begin} to~\ref{line:top-down-end}.
	On each iteration, $\left(\mathbf{L}^{(d)}\right)^\top \vec{y}^{(t)}$ are added, in-place, to $\vec{y}^{(t)}$.
	First, we look at the sparse matrix-vector multiplication, whose work is $\Theta\left(nnz\left(\mathbf{L}^{(d)}\right)\right)$.
	Over the entire for-loop, the total work of the sparse matrix-vector multiplication is $\Theta\left(nnz\left(\mathbf{L}^{(1)}\right)\right) + \hdots + \Theta\left(nnz\left(\mathbf{L}^{(k)}\right)\right) = \Theta(|\mathcal{P}|)$, i.e., linear.
	Next, we analyze the in-place addition, which can be implemented efficiently so the total work is linear with respect to only the number of newly added values, which, in turn, is equal to the number of edges in the level structure at depth $d$.
	The edges don't overlap across levels, so the total work of the in-place additions over the for-loop is proportional to the total number of edges, i.e., $\Theta(|\mathcal{P}|)$.
	In sum, Lines~\ref{line:top-down-begin} to~\ref{line:top-down-end} has the total work of $\Theta(|\mathcal{P}|)$.
	Next, the work in Line~\ref{line:strategy-conversion} is $\Theta(nnz(\mathbf{A})) = \Theta(\Sigma)$.
	Note that $|\Sigma| = O(|\mathcal{P}|)$ and $|\Sigma^+| = O(|\mathcal{P}|)$.
	Overall, the total work is $W = \Theta(|\mathcal{P}|)$.

	Now, we analyze the depth.
	We begin with Line~\ref{line:behavioral-strategy-calculation}.
	The depth of the inner and outer CSR sparse matrix-vector multiplications are $\Theta(\log{(\max_{j \in \mathcal{J}}{|\mathcal{A}_j|})})$ and $\Theta(1)$, respectively, whereas the depth of the Hadamard division is $\Theta(1)$.
	Thus, the depth of Line~\ref{line:behavioral-strategy-calculation} is $\Theta(\log{(\max_{j \in \mathcal{J}}{|\mathcal{A}_j|})})$.
	The depth of Line~\ref{line:unit-vector} is trivially $\Theta(1)$.
	The depth of Line~\ref{line:top-down-end} is also $\Theta(1)$.
	This is repeated $k$ times, so the depth of the for-loop between Lines~\ref{line:top-down-begin} and~\ref{line:top-down-end} is $\Theta(k)$.
	Finally, the depth of Line~\ref{line:strategy-conversion} is $\Theta(1)$.
	Overall, the depth is $D = \Theta(\log{(\max_{j \in \mathcal{J}}{|\mathcal{A}_j|})} + k)$.
\end{proof}

\subsection{Proof of Lemma~\ref{lmm:observe-utility-complexity}}
\label{sec:observe-utility-complexity}

\begin{proof}
	We begin by analyzing the total work.
	Line~\ref{line:zeros} has the work of $\Theta(|\mathcal{P}|)$, and Line~\ref{line:utility-conversion} has the work of $\Theta(nnz(\mathbf{A})) + \Theta(|\mathcal{P}|) = \Theta(|\Sigma|) + \Theta(|\mathcal{P}|) = \Theta(|\mathcal{P}|)$.
	Following the same process used in the proof of Lemma~\ref{lmm:next-strategy-complexity}, we can derive the total work of the for-loop between Lines~\ref{line:bottom-up-begin} and~\ref{line:bottom-up-end} to be $\Theta(|\mathcal{P}|)$.
	Line~\ref{line:combination} has the total work of $\Theta(|\mathcal{P}|) + \Theta(nnz(\mathbf{B})) = \Theta(|\mathcal{P}|) + \Theta(|\Sigma^+|) = \Theta(|\mathcal{P}|)$.
	In Line~\ref{line:regret-update}, the two sparse matrix-vector multiplications have the work of $\Theta(nnz(\mathbf{C})) = \Theta(|\Sigma^+|)$ each, whereas the rest of the operations have the total work of $\Theta(|\Sigma^+|)$.
	This yields the work of $\Theta(|\Sigma^+|)$ for Line~\ref{line:regret-update}.
	Overall, the total work is $W = \Theta(|\mathcal{P}|)$.

	The depths of Lines~\ref{line:zeros} and~\ref{line:utility-conversion} is $\Theta(1)$.
	As in the depth analysis in the proof of Lemma~\ref{lmm:next-strategy-complexity}, the depth of the in-place addition in Line~\ref{line:bottom-up-end} is $\Theta(1)$, and, over the for-loop, $\Theta(k)$.
	The maximum number of non-zero entries in any row in $\mathbf{L}^{(d)}$ for any depth $d$ is bounded by $B$, the degree of the input TFSDP, so the CSR sparse matrix-vector multiplications in Line~\ref{line:bottom-up-end} is $\Theta(\log{B})$, and, over the for-loop, $\Theta(k\log{B})$.
	Thus, Lines~\ref{line:bottom-up-begin} and~\ref{line:bottom-up-end} have the depth of $\Theta(k\log{B})$.
	The depth of Line~\ref{line:combination} is $\Theta(1)$.
	Finally, the depth of Line~\ref{line:regret-update} is $\Theta(\log{(\max_{j \in \mathcal{J}}{|\mathcal{A}_j|})})$.
	Note that $\max_{j \in \mathcal{J}}{|\mathcal{A}_j|} = O(B)$.
	So, the overall depth is $D = \Theta(k\log{B})$.
\end{proof}

\subsection{Proof of Theorem~\ref{thm:complexity}}
\label{sec:complexity}

\begin{proof}
	Using Lemmas~\ref{lmm:next-strategy-complexity} and~\ref{lmm:observe-utility-complexity}, $W = \Theta(|\mathcal{P}|) + \Theta(|\mathcal{P}|) = \Theta(|\mathcal{P}|)$ and $D = \Theta(\log{(\max_{j \in \mathcal{J}}{|\mathcal{A}_j|})} + k) + \Theta(k\log{B}) = \Theta(k\log{B})$, as required.
\end{proof}

\subsection{Proof of Corollary~\ref{cll:parallelism}}
\label{sec:parallelism}

\begin{proof}
	It follows immediately from Theorem~\ref{thm:complexity} and the definition of parallelism.
\end{proof}

\subsection{Proof of Theorem~\ref{thm:space-complexity}}
\label{sec:space-complexity}

\begin{proof}
	Since all vectors and sparse matrices involved as operands require $\Theta(|\mathcal{P}|)$ space, it suffices to show that every linear algebra operation in Algorithm~\ref{alg:parallelized-cfr} has a linear space complexity.
	The space complexity of sparse matrix-vector multiplication is proportional to the number of non-zeros, and every vector operation in Algorithm~\ref{alg:parallelized-cfr} has linear space complexity, as required.
\end{proof}

\section{Testbench specification}
\label{sec:resources}

Our testbench contains an AMD Ryzen 9 3900X 12-core, 24-thread desktop processor, 128 GB memory, and Nvidia GeForce RTX 4090 24 GB VRAM graphics card, containing 16,384 CUDA cores.

\section{Initialization parameters for battleship}
\label{sec:battleship}

\begin{table}[t!]
	\centering
	\caption{
		Initialization parameters of the four battleship games we used in our benchmarks.
		The values shown in the table correspond exactly to the arguments passed to OpenSpiel.
	}
	\label{tab:battleship}
	\begin{tabular}{l|cc|cc|c}
		\toprule
		\multirow{2}{*}{Size} & \multicolumn{2}{c|}{Board} & \multicolumn{2}{c|}{Ship} & \multirow{2}{*}{\# shots} \\
		& Height & Width & Sizes & Values & \\
		\midrule
		Tiny & 2 & 2 & [2] & [1] & 3 \\
		Small & 3 & 2 & [2] & [1] & 3 \\
		Medium & 4 & 4 & [1] & [1] & 2 \\
		Large & 3 & 3 & [1,2] & [1,1] & 2 \\
		\bottomrule
	\end{tabular}
\end{table}

Four battleship games (plus three others) were used during the benchmarks of our parallelized CFR implementation against OpenSpiel's.
Section~\ref{sec:experiments} relegated the exact parameters used to initialize the four battleship games to this appendix section.
Table~\ref{tab:battleship} tabulates these parameters.

\section{Additional experimental results}
\label{sec:additional}

\begin{figure}[t!]
	\centering
	\includegraphics[width=\linewidth]{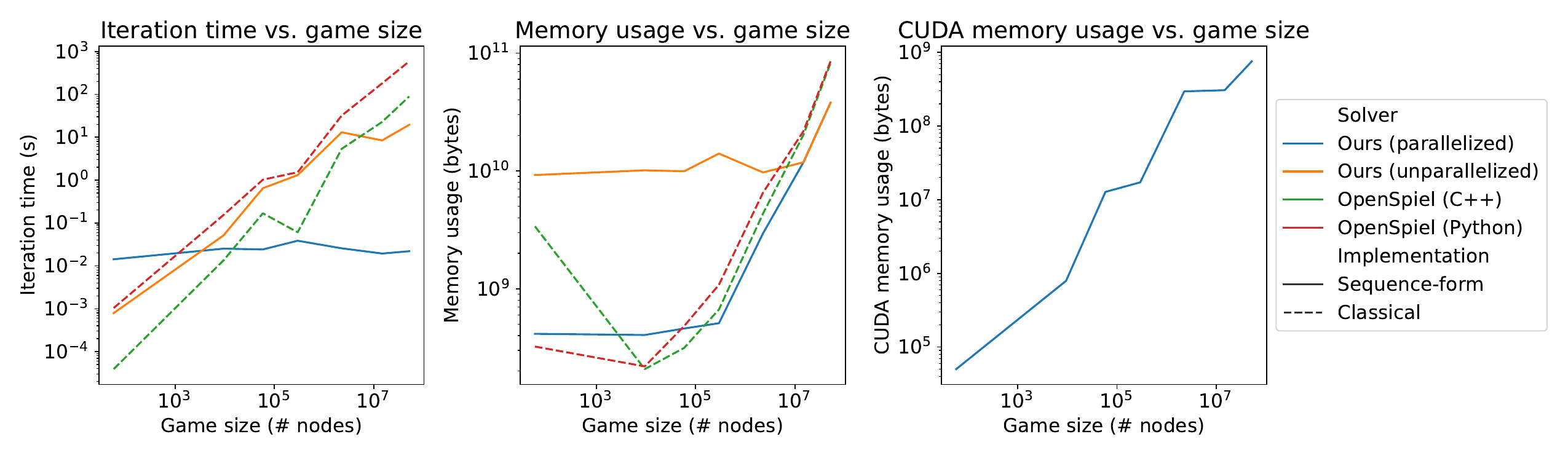}
	\caption{Iteration times and (CUDA) memory usage versus the size of the game being solved.}
	\label{fig:gpugt:plot2}
\end{figure}

\begin{table}[t!]
	\centering
	\caption{
		Iteration runtimes of various CFR implementations.
		For each game, the best performance is bolded.
		Each average iteration time value is accompanied by the standard error of the mean.
	}
	\label{tab:iteration-time}
	\begin{tabular}{l|cccc}
		\toprule
		\multirow{2}{*}{Game} & \multicolumn{4}{c}{Average iteration time} \\
		& \multicolumn{2}{c}{Ours} & \multicolumn{2}{c}{OpenSpiel} \\
		& Parallelized & Unparallelized & C++ & Python \\
		\midrule
		Kuhn poker & $14.3\pm0.1$ \si{\milli\second} & $791\pm1$ \si{\micro\second} & $\mathbf{39\pm1}$ \si{\micro\second} & $1.04\pm0.01$ \si{\milli\second} \\
		Leduc poker & $25.2\pm0.1$ \si{\milli\second} & $51.3\pm0.1$ \si{\milli\second} & $\mathbf{13.4\pm0.1}$ \si{\milli\second} & $154\pm1$ \si{\milli\second} \\
		Tiny battleship & $\mathbf{24.2\pm0.1}$ \si{\milli\second} & $649\pm1$ \si{\milli\second} & $168\pm1$ \si{\milli\second} & $1.03\pm0.01$ \si{\second} \\
		Liar's dice & $\mathbf{38.6\pm0.1}$ \si{\milli\second} & $1.31\pm0.01$ \si{\second} & $60.8\pm0.1$ \si{\milli\second} & $1.53\pm0.01$ \si{\second} \\
		Small battleship & $\mathbf{25.7\pm0.1}$ \si{\milli\second} & $13.1\pm0.1$ \si{\second} & $5.38\pm0.02$ \si{\second} & $32.1\pm0.7$ \si{\second} \\
		Medium battleship & $\mathbf{19.4\pm0.1}$ \si{\milli\second} & $8.45\pm0.04$ \si{\second} & $22.8\pm0.2$ \si{\second} & $181\pm10$ \si{\second} \\
		Large battleship & $\mathbf{22.0\pm0.1}$ \si{\milli\second} & $19.6\pm0.2$ \si{\second} & $89.6\pm4.7$ \si{\second} & $590\pm47$ \si{\second} \\
		\bottomrule
	\end{tabular}
\end{table}

\begin{table}[t!]
	\centering
	\caption{
		On the left, memory usage of the CFR implementations, and on the right, CUDA memory usage of our parallelized implementation.
	}
	\label{tab:memory-usage}
	\resizebox{\linewidth}{!}{
		\begin{tabular}{l|cccc|c}
			\toprule
			\multirow{3}{*}{Game} & \multicolumn{4}{c|}{Memory usage} & \multirow{3}{*}{\makecell{CUDA \\ memory usage}} \\
			& \multicolumn{2}{c}{Ours} & \multicolumn{2}{c|}{OpenSpiel} & \\
			& Parallelized & Unparallelized & \hspace{1.5em} C++ \hspace{1.5em} & \hspace{1em} Python \hspace{1em} & \\
			\midrule
			Kuhn poker & 414 MB & 9.24 GB & 3.39 GB & 323 MB & 50.1 KB \\
			Leduc poker & 405 MB & 10.1 GB & 208 MB & 219 MB & 791 KB \\
			Tiny battleship & 460 MB & 9.94 GB & 315 MB & 483 MB & 12.8 MB \\
			Liar's dice & 510 MB & 14.0 GB & 668 MB & 1.08 GB & 17.2 MB \\
			Small battleship & 2.95 GB & 9.72 GB & 4.34 GB & 6.54 GB & 295 MB \\
			Medium battleship & 11.8 GB & 11.8 GB & 20.2 GB & 21.7 GB & 306 MB \\
			Large battleship & 38.0 GB & 37.9 GB & 82.0 GB & 85.2 GB & 755 MB \\
			\bottomrule
		\end{tabular}
	}
\end{table}

We continue from the remarks we made regarding our experimental results in Section~\ref{sec:experiments}.
Note that the iterates produced by different CFR implementations diverge for some games.
This is primarily due to the following factors: a) regret matching involves mathematical operations with high numerical sensitivity, b) the degree of machine precision differs slightly between a GPU and a CPU, and c) our implementations have different thresholds for checking zero-vectors compared to OpenSpiel's.
This difference is quite striking for the medium battleship game.
While this is not visible in the plot, the peak exploitabilities differ even between the two OpenSpiel implementations by around a factor of three ($0.05$ vs. $0.14$), and the peaks between our implementations also differ by about an order of magnitude.

Figure~\ref{fig:gpugt:plot2} showcases how the iteration times, memory usage, and CUDA memory usage vary as the size of the game being solved increases.
For the iteration times, we again see that our parallelized implementation is slower than the baselines for smaller games, but, as the game size increases, our implementation begins to outperform all other implementations, and the difference grows with the game size.
The performance of our unparallelized implementation sits roughly between those of OpenSpiel's Python and C++ implementations.

The process memory usage and CUDA memory usage (for our parallelized implementation) are tabulated in Table~\ref{tab:memory-usage}.
Note that we kept track of average strategies for a logarithmic number of iterations, so the memory consumed by keeping track of the strategy profiles is also reflected by the memory usage.
Our script's behavior reflects how one would use CFR in real-life applications, as one would typically save the average strategies across some iterations to keep track of the learning process.
Note that our implementations and OpenSpiel represent games and strategy profiles differently, so it is difficult to directly compare the memory usage.
Here, the process memory usage of the implementations ranges from 208 MB to 85.2 GB and is shown to scale with the game size.
Similarly, the (CUDA) memory usage of our parallelized implementation ranges from 50.1 KB to 755 MB and rises roughly linearly with respect to the game size.



\end{document}